\def\year{2020}\relax
\documentclass[letterpaper]{article} 
\usepackage{aaai20}  
\usepackage{times}  
\usepackage{helvet} 
\usepackage{courier}  
\usepackage[hyphens]{url}  
\usepackage{graphicx} 
\usepackage{graphicx}  
\usepackage[fleqn]{amsmath}
\usepackage[font=small,labelfont=bf]{caption}
\usepackage{bm}
\usepackage{amsthm}
\newtheorem{theorem}{Theorem}
\newcommand{\citet}[1]{\citeauthor{#1} \shortcite{#1}}
\usepackage{xcolor}
\title{Hierarchical Average Reward Policy Gradient Algorithms}
\author{Akshay Dharmavaram\\
Birla Institute of Technology and Science, Pilani\\
f20150039@goa.bits-pilani.ac.in\\
 +91-9604079793
\And Matthew Riemer\\
IBM Research\\
mdriemer@us.ibm.com
\And
Shalabh Bhatnagar\\
Indian Institute of Science, Bangalore\\
shalabh@iisc.ac.in
}
 
\begin{document}

\maketitle

\begin{abstract}
Option-critic learning is a general-purpose reinforcement learning (RL) framework that aims to address the issue of long term credit assignment by leveraging temporal abstractions. However, when dealing with extended timescales, discounting future rewards can lead to incorrect credit assignments. In this work, we address this issue by extending the hierarchical option-critic policy gradient theorem for the average reward criterion. Our proposed framework aims to maximize the \textit{long-term} reward obtained in the steady-state of the Markov chain defined by the agent's policy. Furthermore, we use an ordinary differential equation based approach for our convergence analysis and prove that the parameters of the intra-option policies, termination functions, and value functions, converge to their corresponding optimal values, with probability \textit{one}. Finally, we illustrate the competitive advantage of learning options, in the average reward setting, on a grid-world environment with sparse rewards.
\end{abstract}

\section{Introduction}
\noindent Humans routinely employ high-level temporal abstractions for everyday decision making. \citet{oc} investigate the use of learning temporally extended abstractions in order to augment the exploration and credit assignment capabilities of the actor-critic framework. However, employing a discount factor to bound the cumulative rewards can inadvertently lead to incorrect credit assignment. We addresses this issue by extending the framework proposed by \citet{weightsharingAAAI} for the average reward (AR) criterion. 

Figure \ref{fig3}(a) is a motivating example that illustrates how simple \textit{traps} can beguile the discounted rewards (DR) framework into learning a sub-optimal credit assignment. It illustrates two different Markov chains, resulting from two disparate policies ($\pi_{R}$ and $\pi_{B}$). $\pi_{R}$ always chooses red and $\pi_{B}$ always chooses blue. When a DR-RL agent is at $S_{0}$, it has a predilection for the sub-optimal policy $\pi_{B}$, because $\forall \gamma<1$: 
\begin{footnotesize}
\begin{align*}
    v_{\pi_{R}}(S_{11})=\frac{\gamma(2-\gamma)}{(1-\gamma^{4})} < \frac{1}{(1-\gamma^{4})}= v_{\pi_{B}}(S_{21}) 
\end{align*}
\end{footnotesize}

\section{Policy-Gradient with Function Approximation} 

First, we illustrate how to extend the framework proposed by \citet{weightsharingAAAI} for the AR criterion. Apart from addressing the AR criterion, our framework also presents a simplified and intuitive approach to dealing with hierarchical option-critic algorithms \cite{hoc} by introducing the concept of $o^{0}$ and $o^{N}$.

\begin{figure}[t]
\centering
\includegraphics[width=0.9\columnwidth]{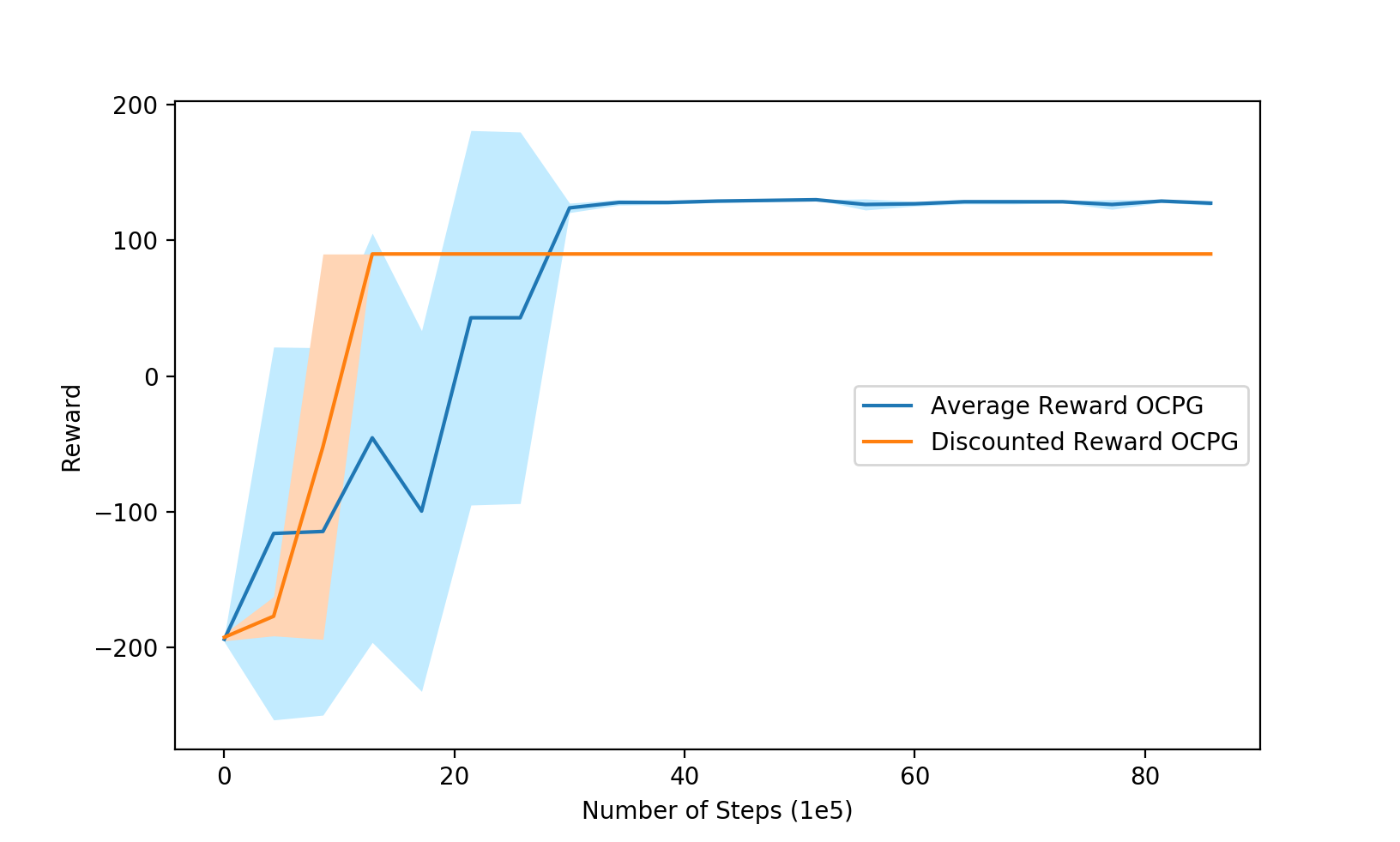} 
\caption{ Mean and standard deviations of the learning curves for the average reward and discounted reward OCPG agents, in the grid-world delivery experiment.}
\label{fig1}
\end{figure}
\begin{theorem}[Hierarchical Average Reward Option-Critic Policy Gradient (OCPG) Theorem]
\label{HARPG}
Given an $N$ level hierarchical set of  Markov  options  with  stochastic  option  policies at each level $\pi^\ell$ and termination functions at each level $\beta^\ell$ differentiable in their parameters $\bm{\theta}$, the gradient of the expected reward per step with respect to $\bm{\theta}$ is:

\begin{footnotesize}
\begin{align*}
    \sum_{s,o^{0:N-1},s'}\!\!\!\!\mu_\Omega(s,o^{0:N-1},s') 
    \bigg(\sum_{a} \frac{\partial \pi(a|s,o^{0:N-1})}{\partial \bm{\theta}} \\
    Q_U(s,o^{0:N-1},a) +  \sum_{o'^{0:N-1}} \sum_{\ell=1}^{N-1} \bigg[ \!\prod_{k=N-1}^{\ell}\!\!\!\!\beta^k(s',o^{0:k}) \\ \frac{\partial \pi^\ell(o'^\ell|s',o'^{0:\ell-1})}{\partial \bm{\theta}} Q_\Omega(s',o'^{0:\ell}) P_{\pi,\beta}(o'^{0:\ell-1}|s',o^{0:\ell-1}) \\
    - \frac{\partial \beta^\ell(s',o^{0:\ell})}{\partial \bm{\theta}} A_\Omega(s',o^{0:\ell})\!\!\!\!\prod_{k=N-1}^{\ell+1}\!\!\!\!\beta^k(s',o^{0:k}) \bigg] \bigg),
\end{align*}
\end{footnotesize}
where $\mu_\Omega$ is the stationary distribution of the Markov chain defined by the hierarchical policy, and $P_{\pi,\beta}$ is the probability while at the next state, and terminating the options for the last state, that the agent arrives at a particular new set of option selections.
\end{theorem}

\begin{proof}
The proof for this theorem is in the Appendix.
\end{proof}

\section{Two-Timescale Convergence}
Next, we prove that the aforementioned parameters, $\theta$, asymptotically converge to their optimal values, when employing a linear approximation $\forall$ $Q_\Omega$. 
We analyze our framework using the ordinary differential equation (ODE) approach, delineated by \citet{bhatnagar2009natural}, and study its asymptotic properties using the fixed points of the derived ODE.


\begin{theorem}[Convergence Proof]
\label{conv}
For the parameter iterations of the global set of shared parameters defined in Algorithm 1, we have ($\hat{J}_{t}, \upsilon_{t}, \theta_{t} $) $\to$ $\{(J(\theta^{*})_{t}, \upsilon^{*}, \theta^{*} )|\theta^{*} \in \mathcal{Z}\}$ as t $\to \infty$ with probability one, where $\mathcal{Z}$ corresponds to the set of local maxima of a performance function whose gradient is $E[\delta^{\pi}_{t}\psi(s_{t},a_{t})|\theta]$
\end{theorem}

\begin{proof}
The proof for this theorem is in the Appendix.
\end{proof}

\begin{figure}[t]
\centering
\includegraphics[width=0.9\columnwidth]{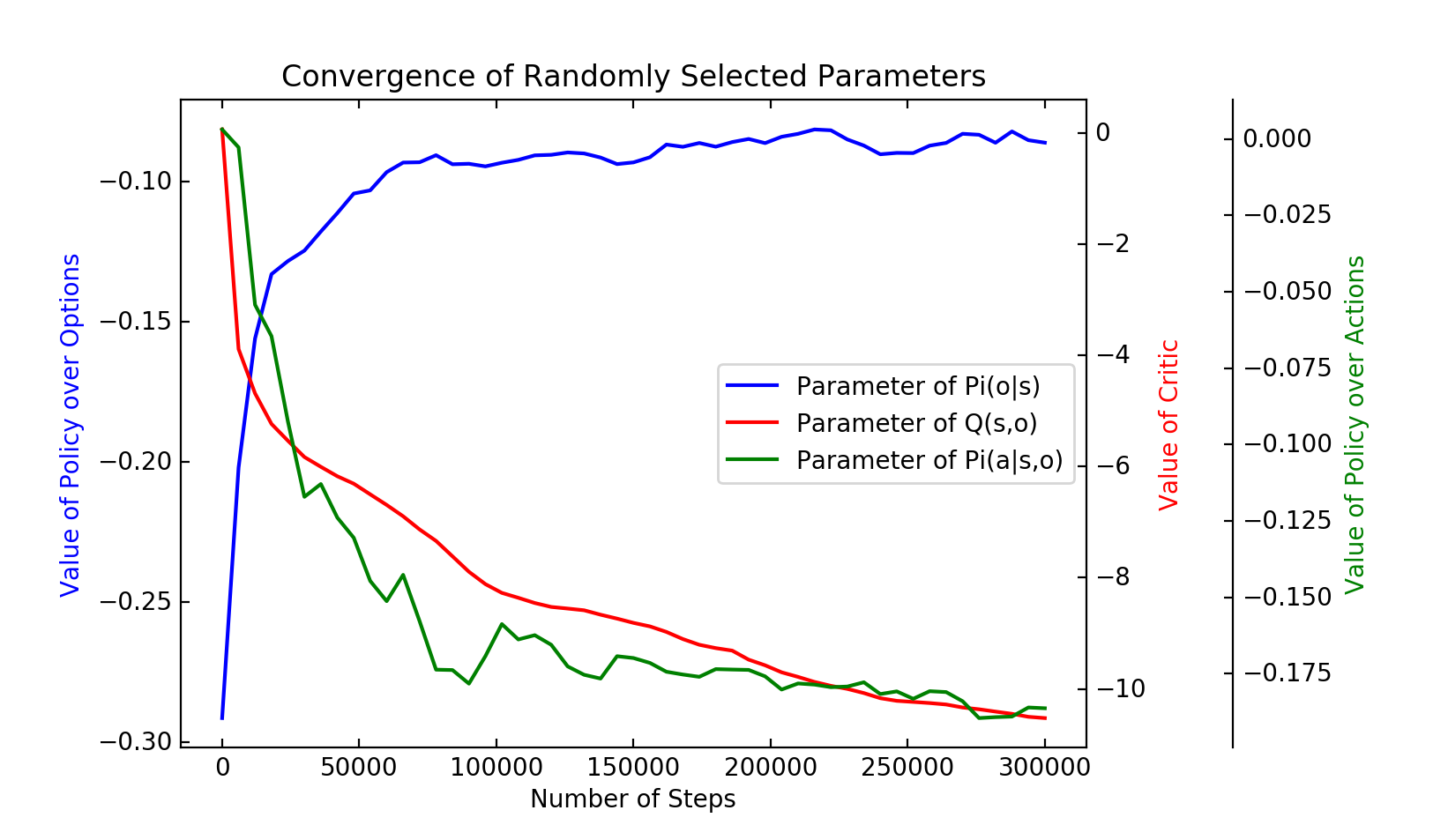} 
\caption{An empirical demonstration illustrating the convergence of the parameters of $Q(s,o)$, $\pi(a|s,o)$, and $\pi(o|s)$. We have randomly selected one parameter from each function approximator and plotted its value against the number of steps.}
\label{fig2}
\end{figure}


\section{Empirical Results}
Finally, we look at the susceptibility of our framework to traps, and compare it to the DR setting proposed by \citet{weightsharingAAAI}. Figure \ref{fig3}(b) depicts a grid world environment characterized by sparse rewards. An agent must navigate to either one of the pickup locations, $P_{1}$ or $P_{2}$, in order to retrieve a parcel; and must subsequently deliver the parcel to the drop off location. The agent gets a reward of +100 for every parcel from $P_{2}$, and +50 for every parcel from $P_{1}$. The optimal policy for an agent would naturally involve picking up the parcels from $P_{2}$. 

We introduce a trap\footnote[1]{The reward of +20 was primarily chosen for illustrating the potential pitfalls when employ a $\gamma \leq 0.9$. Similar traps can be created for any $\gamma \leq 1$.} at the green-blue junction to entice the DR-RL agents into picking up the parcels from $P_{1}$. Once the agent reaches the blue zone, it obtains a reward of +20 as opposed to a reward of +10 at the red-green junction. In Figure \ref{fig1}, we plot the rewards obtained per cycle for both the AR-RL agent and a DR-RL agent, and show that the hierarchical AR policy gradient performs better than its DR counterpart proposed by \citet{weightsharingAAAI}. Finally, we illustrate the asymptotic convergence of the actor and critic parameters in Figure \ref{fig2}. 

\begin{figure}[t]
\centering
\includegraphics[width=\columnwidth]{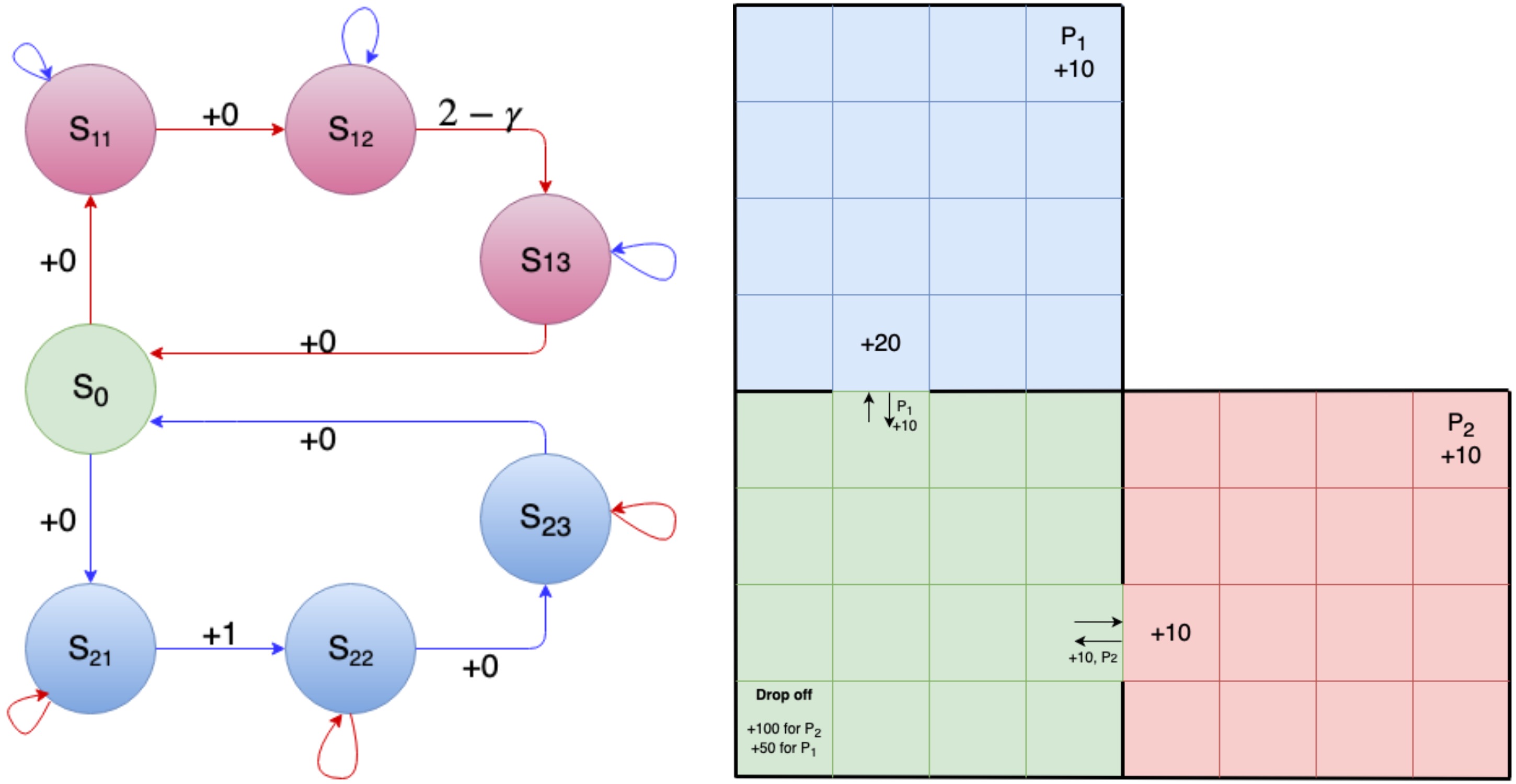} 
\caption{ \textbf{(a)} A \textit{trap} that employs delayed rewards to fool DR-RL agents into learning incorrect credit assignments. \textbf{(b)} A grid-world navigation experiment where the reward at the drop off point depends upon which pickup location was previously visited (50 for $P_1$ and 100 for $P_2$). The trap at the blue-green junction misguides agents towards the sub-optimal pickup location, $P_1$. }
\label{fig3}
\end{figure}

\section{Conclusion and Future Work}
In this work, we propose a novel method for maximizing the long term steady-state reward, by learning intra-option policies, termination functions, and value functions end-to-end. These algorithms can be used in infinite-horizon control problems that exhibit an inherent cyclic structure, like inventory-management, queuing and traffic light control. A detailed empirical analysis for a cyclical infinite-horizon application would be necessary to demonstrate the viability of our approach in complex environments. Additionally, while the proofs provided here leverage a linear approximation for each of the $Q_\Omega(s,o^{0:\ell})$, it would also be interesting to investigate the convergence properties of a non-linear critic.

\bibliographystyle{aaai}
\bibliography{bibliography}

\appendix
\section{Appendix}
Here, we provide the proofs for our theorems \footnote{See a one column format of this paper at }, as well as some extra analysis regarding a selected few topics. \textbf{These proofs are our major contributions, and are placed in the appendix solely due to space constraints.}

\subsection{Options $o^{0}$ and $o^{N}$}
Our hierarchical architecture builds upon the framework proposed by \cite{hoc}. All intra-option policies are of the form $\pi_{\theta^\ell}^\ell(o^{\ell}|s,o^{0:\ell-1})$ $\forall \ell \in$ {1,2,...N}, and each level of the option hierarchy has a complimentary termination function $\beta_{\phi^{\ell}}^{\ell}(s,o^{0:\ell})$. $o^{0}$ unifies under one umbrella: \textbf{(1)} $\pi(o|s)$ and $\pi^{\ell}(o^{\ell}|s,o^{0:\ell-1})$, and \textbf{(2)} $Q_\Omega(s,o^{0:\ell-1})$ and $V_\Omega(s')$; which were considered as disparate terms in prior work. 

$o^N$ corresponds to the primitive actions, and intuitively follows a termination policy with $\beta_{\phi^{N}}^{N}(s,o^{0:N})=1$. On the other hand, $o^0$ corresponds to a \textit{super-option}, and follows a termination policy with $\beta_{\phi^{0}}^{0}(s,o^{0})=0$. Instead of visualizing the agent as an external entity picking a starting option, $o^{1}$, the agent can be thought of as executing a \textit{super-option} which never terminates. Apart from leading to shorter equations and proofs, this framework naturally leads to the idea of \textit{stacking} option-hierarchies, and the intuition that the agent is part of a deeper network of hierarchies. This approach could lead to novel avenues of research.

We make the following changes to the HOC framework:
\textbf{(1)} $Q_U: \mathcal{S} \times \Omega \times \mathcal{A} \rightarrow \mathcal{R}$ is redefined to address the average reward criterion:

\vspace{-5mm}
\begin{equation}\label{eq:QU}
\begin{aligned}
Q_U(s,o^{1:\ell}) = \sum_{o^{\ell+1:N}}\prod_{j=\ell+1}^{N}\pi^{j}(o^{j}|s,o^{1:j-1})\Bigg[r(s,o^{N}) \\
- J(\pi) + \sum_{s'} P(s'|s,o^{N})U(s',o^{1:\ell-1})\Bigg]
\end{aligned}
\end{equation}
\vspace{-4mm}

\textbf{(2)} The new unified upon-arrival value-function presented below has two terms, instead of four.

\begin{align*}
U(s',o^{0:\ell-1}) =\\ \underbrace{Q_\Omega(s',o^{0:\ell-1}) \sum_{q=\ell-1}^{N-1}(1-\beta^{q}(s',o^{0:q}))  \prod_{k=q+1}^{N} \beta^k(s',o^{0:k}) }_{\text{lower level options terminate }}\\
+ \underbrace{ \sum_{i=0}^{\ell-2}(1-  \beta^i(s',o^{0:i})) Q_\Omega(s',o^{0:i}) \prod_{k=i+1}^{N} \beta^k(s',o^{0:k}) }_{\text{higher level options terminate}}
\end{align*}

\subsection{Hierarchical Average Reward Policy Gradient} \label{HOCPGProof}

We begin by presenting a few preliminary expressions, based on the theorems delineated by \cite{hoc}, which form the basis of our subsequent proofs. 

If $o_t^{0:N-1}$ is executing at time $t$, then the discounted probability of transitioning to $(s_{t+1},o_{t+1}^{0:N-1})$ is:

\begin{equation} 
\begin{split}
P^{(1)}(s_{t+1},o_{t+1}^{0:N-1}|s_t,o_t^{0:N-1}) =\\
\sum_a \pi(a|s_t,o_t^{0:N-1})  P(s_{t+1}|s_t,a) \bigg( \\ (1 - \beta^{N-1}(s_{t+1},o_t^{0:N-1}))\textbf{1}_{o_{t+1}^{0:N-1} =o_t^{0:N-1}} \\ +
\sum_{i=0}^{N - 2}(1-\beta^i(s',o^{0:i})) \textbf{1}_{o_{t+1}^{0:i}=o_t^{0:i}} \bigg[\\
\prod_{k=N-1}^{i+1} \beta^k(s',o_t^{0:k}) \pi^k(o_{t+1}^k|s_{t+1},o_{t+1}^{0:k-1}) \bigg] \bigg).
\end{split}
\end{equation}

The discounted probabilities for k-steps can more generally be expressed recursively: 

\begin{equation} 
\begin{split}
P^{(k)}(s_{t+k},o_{t+k}^{0:N-1}|s_t,o_{t}^{0:N-1}) = \\ \sum_{s_{t+1},o_{t+1}^{0:N-1}} \bigg( P^{(1)}(s_{t+1},o_{t+1}^{0:N-1}|s_t,o_{t}^{0:N-1}) \\ P^{(k-1)}(s_{t+k},o_{t+k}^{0:N-1}|s_{t+1},o_{t+1}^{0:N-1}) \bigg).
\end{split}
\end{equation}

Next, we define $Q_\Omega$, and $U$ for state $s$ and active options $o^{0:N-1}$ directly following \cite{hoc}. The option-value function $Q_\Omega$ can be expressed as:

\begin{equation} 
Q_\Omega(s,o^{0:N-1}) =\\
\sum_{a} \pi^N(a|s,o^{0:N-1})Q_U(s,o^{0:N-1},a).
\end{equation}

We incorporate the average reward optimality criterion into the definition of  $Q_U$, the value of executing an action in the presence of the currently active options, as: 

\begin{equation} 
\begin{split}
Q_U(s,o^{0:N-1},a) = r(s,a) - J(\pi) + \\
\sum_{s'} P (s'|s,a)U(s',o^{0:N-1}).
\end{split}
\end{equation}

We also follow the option value function upon arrival $U$ from \cite{hoc}. 

\begin{align*}
U(s',o^{0:N-1}) =\\
\sum_{i=0}^{N-1}(1-  \beta^i(s',o^{0:i})) Q_\Omega(s',o^{0:i}) \prod_{k=i+1}^{N} \beta^k(s',o^{0:k}) 
\end{align*}

We can now follow a similar procedure as the one explored \cite{weightsharingAAAI}, and take the derivative of $Q_\Omega(s,o^{0:N-1})$ with respect to $\bm{\theta}$. By incorporating the notion of $o^{0}$ into our equations, we were able to significantly reduce the complexity of the following equations. 

\begin{equation}  \label{QOmega}
\begin{split}
\frac{\partial Q_\Omega(s,o^{0:N-1})}{\partial \bm{\theta}} =\\
\frac{\partial}{\partial \bm{\theta}}\sum_{a} \pi^N(a|s,o^{0:N-1})Q_U(s,o^{0:N-1},a) \\ 
= \sum_{a} \frac{ \partial \pi^N(a|s,o^{0:N-1})}{\partial \bm{\theta}} Q_U(s,o^{0:N-1},a)
+ \\
\sum_{a} \pi^N(a|s,o^{0:N-1}) \frac{\partial Q_U(s,o^{0:N-1},a)}{\partial \bm{\theta}},
\end{split}
\end{equation}

Next, we take the derivative of $Q_U(s,o^{0:N-1},a)$ and $U(s',o^{0:N-1})$ with respect to $\bm{\theta}$:

\begin{equation} 
\begin{split}
\frac{ \partial Q_U(s,o^{0:N-1},a)}{\partial \bm{\theta}} =  \sum_{s'} P (s'|s,a) \frac{\partial U(s',o^{0:N-1})}{\partial \bm{\theta}}\\
- \frac{\partial J(\pi)}{\partial \bm{\theta}},
\end{split}
\end{equation}

\begin{equation} \label{eq31}
\begin{split}
\frac{\partial U(s',o^{0:N-1})}{\partial \bm{\theta}} = \sum_{i=0}^{N - 1}(1-\beta^i(s',o^{0:i})) Q_\Omega(s',o^{0:i}) \bigg(\\
\sum_{j=i+1}^{N-1} \frac{ \partial \beta^j(s',o^{0:j})}{\partial \bm{\theta}} \prod_{\substack{k=N \\ k \neq j}}^{i+1} \beta^k(s',o^{0:k}) \bigg)\\
- \sum_{i=0}^{N - 1}\frac{ \partial \beta^i(s',o^{0:i})}{\partial \bm{\theta}} Q_\Omega(s',o^{0:i}) \prod_{k=N}^{i+1} \beta^k(s',o^{0:k}) \\
+ \sum_{i=0}^{N - 1}(1-\beta^i(s',o^{0:i})) \frac{\partial Q_\Omega(s',o^{0:i})}{\partial \bm{\theta}} \prod_{k=N}^{i+1} \beta^k(s',o^{0:k}).
\end{split}
\end{equation}

Likewise, we can define the option-value function $Q_\Omega(s,o^{0:i})$ by integrating out the option-value function using the policy over options at each layer:

\begin{equation} 
Q_\Omega(s,o^{0:i}) = \sum_{o'^{i+1:N-1}} \prod_{i=i+1}^{N-1} \pi^i(o^i|s,o^{0:i-1}) Q_\Omega(s,o^{0:N-1}).
\end{equation}

We now take the gradient to obtain:

\begin{equation}   \label{eq35}
\begin{split}
\frac{\partial Q_\Omega(s,o^{0:i})}{\partial \bm{\theta}} = \sum_{o^{i+1:N-1}} \bigg( \prod_{j=i+1}^{N-1} \pi^j(o^j|s,o^{0:j-1}) \frac{\partial Q_\Omega(s,o^{0:N-1})}{\partial \bm{\theta}} \\
+ \sum_{j=i+1}^{N-1} \frac{ \partial \pi^j(o^j|s,o^{0:j-1})}{\partial \bm{\theta}} \prod_{\substack{k=i+1 \\ k \neq j}}^{N-1} \pi^k(o^k|s',o^{0:k-1}) Q_\Omega(s,o^{0:N-1}) \bigg).
\end{split}
\end{equation}

We can now simplify our original expression of $U(s',o^{0:N-1})$ by substituting the values of the.

\begin{footnotesize}
\begin{equation} 
\begin{split}
\frac{\partial U(s',o^{0:N-1})}{\partial \bm{\theta}} = \bigg( (1-\beta^{N-1}(s',o^{0:N-1})) \textbf{1}_{o'^{0:N-1}=o^{0:N-1}} \\ 
+ \sum_{i=0}^{N - 2}(1-\beta^i(s',o^{0:i})) \textbf{1}_{o'^{0:i}=o^{0:i}} \bigg[ \\ \prod_{k=N-1}^{i+1} \beta^k(s',o^{0:k}) \pi^k(o'^k|s',o'^{0:k-1}) \bigg] \bigg) \frac{ \partial Q_\Omega(s',o'^{0:N-1})}{\partial \bm{\theta}}\\
+ \sum_{i=0}^{N - 1}(1-\beta^i(s',o^{0:i})) \prod_{k=N}^{i+1} \beta^k(s',o^{0:k}) \sum_{o^{i+1:N-1}}  \sum_{j=i+1}^{N-1} \bigg[\\ \frac{ \partial \pi^j(o^j|s',o^{0:j-1})}{\partial \bm{\theta}} \prod_{\substack{k=i+1 \\ k \neq j}}^{N} \pi^k(o^k|s',o^{0:k-1}) Q_\Omega(s',o^{0:N-1}) \bigg]\\
- \sum_{i=0}^{N - 1}\frac{ \partial \beta^i(s',o^{0:i})}{\partial \bm{\theta}} Q_\Omega(s',o^{0:i}) \prod_{k=N}^{i+1} \beta^k(s',o^{0:k}) \\
+ \sum_{i=0}^{N - 1}(1-\beta^i(s',o^{0:i})) Q_\Omega(s',o^{0:i}) \bigg(\\
\sum_{j=i+1}^{N-1} \frac{ \partial \beta^j(s',o^{0:j})}{\partial \bm{\theta}} \prod_{\substack{k=N \\ k \neq j}}^{i+1} \beta^k(s',o^{0:k}) \bigg).
\end{split}
\end{equation}
\end{footnotesize}

We now substitute this last expression into equation \ref{QOmega}:

\begin{equation} 
\begin{split}
\frac{\partial Q_\Omega(s,o^{0:N-1})}{\partial \bm{\theta}} + \frac{\partial J(\pi)}{\partial \bm{\theta}}\\
= \sum_{a} \frac{ \partial \pi^N(a|s,o^{0:N-1})}{\partial \bm{\theta}} Q_U(s,o^{0:N-1},a) \\
+  \sum_{s'} \sum_{o'^{0:N-1}} P^{(1)}(s',o'^{0:N-1}|s,o^{0:N-1}) \frac{\partial Q_\Omega(s',o'^{0:N-1})}{\partial \bm{\theta}} \\
+ \sum_{a} \pi^N(a|s,o^{0:N-1})  \sum_{s'} P (s'|s,a) \bigg[ \\
\sum_{i=0}^{N - 1}(1-\beta^i(s',o^{0:i})) \prod_{k=N}^{i+1} \beta^k(s',o^{0:k}) \sum_{o^{i+1:N-1}}  \sum_{j=i+1}^{N-1} \bigg[\\ \frac{ \partial \pi^j(o^j|s',o^{0:j-1})}{\partial \bm{\theta}} \prod_{\substack{k=i+1 \\ k \neq j}}^{N} \pi^k(o^k|s',o^{0:k-1}) Q_\Omega(s',o^{0:N-1}) \bigg]\\
- \sum_{i=0}^{N - 1}\frac{ \partial \beta^i(s',o^{0:i})}{\partial \bm{\theta}} Q_\Omega(s',o^{0:i}) \prod_{k=N}^{i+1} \beta^k(s',o^{0:k}) \\
+ \sum_{i=0}^{N - 1}(1-\beta^i(s',o^{0:i})) Q_\Omega(s',o^{0:i}) \bigg(\\
\sum_{j=i+1}^{N-1} \frac{ \partial \beta^j(s',o^{0:j})}{\partial \bm{\theta}} \prod_{\substack{k=N \\ k \neq j}}^{i+1} \beta^k(s',o^{0:k}) \bigg)
\bigg],
\end{split}
\end{equation}

As in \cite{hoc} we can further condense our expression by noting that the generalized advantage function over a hierarchical set of options can be defined as $A_\Omega(s',o^{0:\ell}) = Q_\Omega(s',o^{0:\ell})  - \sum_{i=0}^{\ell - 1}(1-\beta^i(s',o^{0:i})) Q_\Omega(s',o^{0:i}) [\prod_{k=i+1}^{\ell - 1} \beta^k(s',o^{0:k})]$. We replace the previous terms into the previous equation.

\begin{equation} 
\begin{split}
\frac{\partial Q_\Omega(s,o^{0:N-1})}{\partial \bm{\theta}} + \frac{\partial J(\pi)}{\partial \bm{\theta}}\\
= \sum_{a} \frac{ \partial \pi^N(a|s,o^{0:N-1})}{\partial \bm{\theta}} Q_U(s,o^{0:N-1},a) \\
+  \sum_{s'} \sum_{o'^{0:N-1}} P^{(1)}(s',o'^{0:N-1}|s,o^{0:N-1}) \frac{\partial Q_\Omega(s',o'^{0:N-1})}{\partial \bm{\theta}} \\
+ \sum_{a} \pi^N(a|s,o^{0:N-1})  \sum_{s'} P (s'|s,a) \bigg[ \\
\sum_{i=0}^{N - 1}(1-\beta^i(s',o^{0:i})) \prod_{k=N}^{i+1} \beta^k(s',o^{0:k}) \sum_{o^{i+1:N-1}}  \sum_{j=i+1}^{N-1} \bigg[\\ \frac{ \partial \pi^j(o^j|s',o^{0:j-1})}{\partial \bm{\theta}} \prod_{\substack{k=i+1 \\ k \neq j}}^{N} \pi^k(o^k|s',o^{0:k-1}) Q_\Omega(s',o^{0:N-1}) \bigg]\\
- \sum_{i=0}^{N - 1}\frac{ \partial \beta^i(s',o^{0:i})}{\partial \bm{\theta}} A_\Omega(s',o^{0:i}) \prod_{k=N}^{i+1} \beta^k(s',o^{0:k}) \bigg],
\end{split}
\end{equation}

We can also condense the terms related to the gradient of $\pi^\ell$ as delineated in \cite{weightsharingAAAI}:

\begin{equation} 
\begin{split}
\frac{\partial Q_\Omega(s,o^{0:N-1})}{\partial \bm{\theta}} + \frac{\partial J(\pi)}{\partial \bm{\theta}}\\
= \sum_{a} \frac{ \partial \pi^N(a|s,o^{0:N-1})}{\partial \bm{\theta}} Q_U(s,o^{0:N-1},a) \\
+  \sum_{s'} \sum_{o'^{0:N-1}} P^{(1)}(s',o'^{0:N-1}|s,o^{0:N-1}) \frac{\partial Q_\Omega(s',o'^{0:N-1})}{\partial \bm{\theta}} \\
+ \sum_{s'} P(s'|s,o^{0:N-1}) \bigg[
\sum_{o^{0:N-1}}  \sum_{j=1}^{N-1}  \frac{ \partial \pi^j(o^j|s',o^{0:j-1})}{\partial \bm{\theta}}\\ Q_\Omega(s',o^{0:N-1}) [\prod_{k=N}^{\ell} \beta^k(s',o^{0:k})] P_{\pi,\beta}(o'^{0:\ell-1}|s',o^{0:\ell-1})\\
- \sum_{i=0}^{N - 1}\frac{ \partial \beta^i(s',o^{0:i})}{\partial \bm{\theta}} A_\Omega(s',o^{0:i}) \prod_{k=N}^{i+1} \beta^k(s',o^{0:k}) \bigg],
\end{split}
\end{equation}

We rearrange and multiply both sides with the stationary distribution $\sum_{s'} \sum_{o'^{0:N-1}}  d_\pi(s',o'^{0:N-1}) $, and cancel the $\sum_{s'} \sum_{o'^{0:N-1}}  d_\pi(s',o'^{0:N-1}) \frac{\partial Q_\Omega(s,o^{0:N-1})}{\partial \bm{\theta}}$ terms on both sides:

\begin{equation}
\begin{split}
\frac{\partial J(\pi)}{\partial \bm{\theta}}
=\sum_{s'} \sum_{o'^{0:N-1}}  d_\pi(s',o'^{0:N-1}) 
\sum_{s'} P(s'|s,o^{0:N-1}) \bigg[ \\
\sum_{a} \frac{ \partial \pi^N(a|s,o^{0:N-1})}{\partial \bm{\theta}} Q_U(s,o^{0:N-1},a)\\
- \sum_{\ell=1}^{N-1} \frac{\partial \beta^\ell(s',o^{0:\ell})}{\partial \bm{\theta}} A_\Omega(s',o^{0:\ell}) [ \prod_{k=N}^{\ell+1} \beta^k(s',o^{0:k}) ] \\
+   \sum_{o'^{0:N-1}} \sum_{\ell=1}^{N-1} \frac{\partial \pi^\ell(o'^\ell|s',o'^{0:\ell-1})}{\partial \bm{\theta}} Q_\Omega(s',o'^{0:\ell}) \\
[\prod_{k=N}^{\ell} \beta^k(s',o^{0:k}) ]  P_{\pi,\beta}(o'^{0:\ell-1}|s',o^{0:\ell-1})
\bigg].
\end{split}
\end{equation}

Finally, we define $\mu_\Omega$ as a discounted weighting of augmented state tuples along steady state trajectories: $\mu_\Omega(s,o^{0:N-1},s') = d_\pi(s',o'^{0:N-1})  P(s_t=s, o_t^{0:N-1}=o^{0:N-1},s_{t+1}=s'|s,o_0^{0:N-1})$. $P_{\pi,\beta}(o'^{0:\ell-1}|s',o^{0:\ell-1})$ is the probability while at the next state and terminating the options for the last state that the agent arrives at a particular set of next option selections.

\begin{equation}
\begin{split}
\frac{\partial J(\pi)}{\partial \bm{\theta}} 
=\sum_{s,o^{0:N-1},s'}  \mu_\Omega(s,o^{0:N-1},s')  \bigg(\\
\sum_{a} \frac{\partial \pi(a|s,o^{0:N-1})}{\partial \bm{\theta}} Q_U(s,o^{0:N-1},a) \\ 
+ \sum_{o'^{0:N-1}} \sum_{\ell=1}^{N-1} \frac{\partial \pi^\ell(o'^\ell|s',o'^{0:\ell-1}) }{\partial \bm{\theta}} Q_\Omega(s',o'^{0:\ell})\\
\prod_{k=N-1}^{\ell} \beta^k(s',o^{0:k})  P_{\pi,\beta}(o'^{0:\ell-1}|s',o^{0:\ell-1}) \\
- \sum_{\ell=1}^{N-1} \frac{\partial \beta^\ell(s',o^{0:\ell})}{\partial \bm{\theta}} A_\Omega(s',o^{0:\ell}) \prod_{k=N-1}^{\ell+1} \beta^k(s',o^{0:k}) \bigg).
\end{split}
\end{equation}

\section{Convergence Proofs}
\cite{bhatnagar2009natural} to proved that if the following three conditions are true, then we can prove that the coupled stochastic recursion will converge with probability one.

\begin{enumerate}
  \item $\sup_{t} ||X_{t}||$, $\sup_{t} ||Y_{t}|| \leq \infty$
  \item  $\Dot{X}=f(X(t),Y)$ has a globally asymptotically stable equilibrium $\mu(Y)$ where $\mu(.)$ is a Lipschitz continuous function.
  \item The slower timescale, ie. Y, assumes that the faster timescale converges quicker in comparison. Thus, it can be modelled as an equation where the X values instantaneously converge to the corresponding values dictated by Y. The corresponding equation $\Dot{Y}=g(\mu(Y(t)),Y(t))$ should have a globally asymptotically stable equilibrium Y*.
\end{enumerate}

We investigate the convergence proofs for the case of N=2; however, the arguments made below can be easily adapted to any N. Let $F_{2}(t) = \sigma(\theta_{r}, r \leq t)$ denote the sequence of $\sigma$-fields generated by $\theta_{r}$, $ r \geq 0$. $\Psi^{\pi}$ is the sequence the value of $\Psi$ when the critic converges to the actor's policy. The update equation is
\begin{equation} 
\begin{split}
  \Psi &= Q_U(s_{t},o_{t},a_{t}) \psi_{s,o,a} -
 A_\Omega(s_{t+1},o_{t}) \psi_{\beta} +\\
 &Q_\Omega(s_{t+1},o_{t+1}) \beta(s_{t+1},o_{t}) \psi_{s,o}  \\
\theta_{t+1} &= \Gamma \bigg[\theta_{t} + \alpha_{t} \Psi
 \bigg]
\end{split}
\end{equation}

Which we will write as 
\begin{equation} 
\begin{split}
\theta_{t+1} = \Gamma \bigg[\theta_{t} + \alpha_{t} E[\Psi^{\pi}| F_{2}(t)] +\\
\alpha_{t} (\Psi - E[\Psi| F_{2}(t)]) + \alpha_{t}E[\Psi-\Psi^{\pi}| F_{2}(t)]
 \bigg]
\end{split}
\end{equation}
Since the critic converges faster, we can claim $E[\Psi-\Psi^{\pi}| F_{2}(t)] = o(1)$. Now, let

\begin{equation} 
\begin{split}
\mathcal{M}^{2}(t) = \sum_{r=0}^{t-1}  \alpha_{r} (\Psi - E[\Psi| F_{2}(t)]), t \geq 1
\end{split}
\end{equation}
Using the assumptions (1-3), we can conclude that the martingale sequence {$\mathcal{M}^{2}(t)$} is convergent. Thus, using the results from the martingale theory, for any $T > 0$, with $n_{T} = min{\sum_{r=0}^{t-1}  \alpha_{r} \geq T}$, we have that $ \sum_{r=0}^{t-1}  \alpha_{r} (\Psi - E[\Psi| F_{2}(t)]) \to 0$ asymptotically as $n \to \infty$. 

Next, we want to simplify $ E[\Psi^{\pi}| F_{2}(t)]$.

\begin{equation} 
\begin{split}
E [ \Psi^{\pi}| F_{2}(t)] = E \bigg[ Q_U(s_{t},o_{t},a_{t}) \psi_{s,o,a} -
 A_\Omega(s_{t+1},o_{t}) \psi_{\beta} + \\
 Q_\Omega(s_{t+1},o_{t+1}) \beta(s_{t+1},o_{t}) \psi_{s,o}| F_{2}(t) \bigg] \\
 = \sum_{s_{t},o_{t}} d^{\pi_{t}}(s,o) \sum_{a_{t}} \bigg( Q_U(s_{t},o_{t},a_{t}) \psi_{s,o,a} -
 A_\Omega(s_{t+1},o_{t}) \psi_{\beta} + \\
 Q_\Omega(s_{t+1},o_{t+1}) \beta(s_{t+1},o_{t}) \psi_{s,o}| F_{2}(t) \bigg); \text{where} \\
 \psi_{\beta} = \nabla log( \beta(s,o)) \\
\psi_{s,o} = \nabla log(\pi_\Omega(o|s)) \\
\psi_{s,o,a} =\nabla log(\pi(a|s,o)) 
\end{split}
\end{equation}

We now define $e^{\pi_{t}}$ using the following equation:

\begin{equation} 
\begin{split}
h(\theta_{t})
 &= E \bigg[ Q_U(s_{t},o_{t},a_{t}) \psi_{s,o,a} -
 A_\Omega(s_{t+1},o_{t}) \psi_{\beta} \\+ &
 Q_\Omega(s_{t+1},o_{t+1}) \beta(s_{t+1},o_{t}) \psi_{s,o}| F_{2}(t) \bigg]\\
&= \nabla J(\pi) + e^{\pi_{t}}
\end{split}
\end{equation}
Now, we have to prove that $h(\theta_{t})$ is Lipschitz continuous by showing that $\exists M$ such that $ |\nabla h(\theta_{t})| < M $. This is where the linear approximations prove to be essential. If we take the linear approximation of the $A_\Omega(s_{t+1},o_{t})$ and $ Q_\Omega(s_{t+1},o_{t+1})$, then we can clearly put a max bound on them. Any matrix can be given an upper bound by simply making all the elements equal to the max element of the matrix. It can similarly be proven for all the other terms that they have bounded derivatives and are continuously differentiable. Thus, $h(\theta_{t})$ is Lipschitz continuous and the ODE is well posed.

Let $n(t) = \sum_{r=0}^{t-1}\beta_{r}$, $t \geq 1$ with n(0) = 0. Let $I_{t} = [n(t), n(t+1)]$, $t \geq 0$. Let $\Bar{\theta}(s), s \geq 0$, be a continuous linear interpolation of iterates of $\theta_{t}$ over the interval $I_{t}$. Using Gronwall's theorem we can show that for any $\Delta>0$, $\exists s(\Delta)>0$ such that $\Bar{\theta}(s(\Delta)+.)$ is a (T, $\Delta$) perturbation.

Let $\sup_{\pi_{t}} || e^{\pi_{t}} || < \delta $ for a $\delta > 0$. Let $\theta^{s(\Delta)}(t)$, $\hat{\theta}^{s(\Delta)}(t)$ be the solutions of $\Dot{\theta}=\hat{\Gamma}(\nabla J(\pi) + e^{\pi})$, $\Dot{\theta}=\hat{\Gamma}(\nabla J(\pi)) $, respectively, for $t \in [s(\Delta),s(\Delta)+T]$, for a given T > 0, with $\theta^{s(\Delta)}(t)$ = $\hat{\theta}^{s(\Delta)}(t)$ = $\Bar{\theta}(s(\Delta))$. From the foregoing, we have $\sup_{t \in [s(\Delta),s(\Delta)+T]} || \theta^{s(\Delta)}(t) - \Bar{\theta}(t) ||< \Delta $. The trajectories $\theta^{s(\Delta)}(t)$ and $\hat{\theta}^{s(\Delta)}(t)$ of their corresponding ODE's are obtained from integration. If we integrate and subtract, we get:

\begin{equation}
\begin{split}
|| \theta^{s(\Delta)}(t) - \hat{\theta}^{s(\Delta)}(t) || \leq \sup_{\pi_{s}} ||e^{\pi}_{s}||(t-s(\Delta)) \leq T\delta
\end{split}
\end{equation}

Hence we can show that:
\begin{equation}
\begin{split}
\sup_{t \in [s(\Delta),s(\Delta)+T]} || \hat{\theta}^{s(\Delta)}(t) - \Bar{\theta}(t) || &\leq \sup_{t \in [s(\Delta),s(\Delta)+T]} || \theta^{s(\Delta)}(t) - \Bar{\theta}(t) || \\ &+ \sup_{t \in [s(\Delta),s(\Delta)+T]} || \hat{\theta}^{s(\Delta)}(t) - \theta(t) ||\\
&\leq \Delta + T\delta
\end{split}
\end{equation}

Thus, $\Bar{\theta}(s(\Delta)+.)$ is a (T, $\Delta$) perturbation of $\Dot{\theta}=\hat{\Gamma}(\nabla J(\pi)) $.

Finally, we use lemma 6 from \cite{bhatnagar2009natural} to conclude that the coupled stochastic recursions converge with probability one.

\end{document}